\documentclass[letterpaper, 10 pt, journal, twoside]{IEEEtran}

\usepackage{array}
\usepackage{textcomp}
\usepackage{stfloats}
\usepackage{url}
\usepackage{verbatim}
\usepackage{cite}

\def\BibTeX{{\rm B\kern-.05em{\sc i\kern-.025em b}\kern-.08em
    T\kern-.1667em\lower.7ex\hbox{E}\kern-.125emX}}
\usepackage{balance}

\usepackage{multirow}
\usepackage{graphicx} 

\usepackage{caption}
\usepackage{subcaption}

\usepackage{xcolor}
\usepackage{algorithm}
\usepackage[noend]{algpseudocode} 
\usepackage{amssymb,amsmath,comment,cite}
\newtheorem{theorem}{Theorem}
\newtheorem{lemma}{Lemma}
\newtheorem{definition}{Definition}
\newtheorem{proof}{Proof}
\newtheorem{corollary}{Corollary}

\usepackage{ifthen}
\newboolean{shortver}
\setboolean{shortver}{false}

\begin{document}

\title{
	A Conflict-Based Search Framework for Multi-Objective Multi-Agent Path Finding
}
\author{Zhongqiang Ren$^{1}$, Sivakumar Rathinam$^{2}$ and Howie Choset$^{1}$
\thanks{Manuscript received November, 30, 2021; Revised April, 28, 2022; Accepted June, 11, 2022. This paper was recommended for publication by Editor Jing Li upon evaluation of the Associate Editor and Reviewers' comments. This work was supported by National Science Foundation under Grant No. 2120219 and 2120529. \textit{(Corresponding author: Zhongqiang Ren.)}
}
\thanks{Zhongqiang Ren and Howie Choset are with Carnegie Mellon University, 5000 Forbes Ave., Pittsburgh, PA 15213, USA. (email: zhongqir@andrew.cmu.edu; choset@andrew.cmu.edu).}%
\thanks{Sivakumar Rathinam is with Texas A\&M University, College Station, TX 77843-3123. (email: srathinam@tamu.edu).}
}


\maketitle


\begin{abstract}
	Conventional multi-agent path planners typically compute an ensemble of paths while optimizing a single objective, such as path length. However, many applications may require multiple objectives, say fuel consumption and completion time, to be simultaneously optimized during planning and these criteria may not be readily compared and sometimes lie in competition with each other. The goal of the problem is thus to find a Pareto-optimal set of solutions instead of a single optimal solution. Naively applying existing multi-objective search algorithms, such as multi-objective A* (MOA*), to multi-agent path finding may prove to be inefficient as the dimensionality of the search space grows exponentially with the number of agents. This article presents an approach named Multi-Objective Conflict-Based Search (MO-CBS) that attempts to address this so-called curse of dimensionality by leveraging prior Conflict-Based Search (CBS), a well-known algorithm for single-objective multi-agent path finding, and principles of dominance from multi-objective optimization literature. We also develop several variants of MO-CBS to improve its performance. We prove that MO-CBS and its variants can compute the entire Pareto-optimal set. Numerical results show that MO-CBS outperforms MOM*, a recently developed state-of-the-art multi-objective multi-agent planner.

\end{abstract}

\def\abstractname{Note to Practitioners}
\begin{abstract}
	The motivation of this article originates from the need to optimize multiple path criteria when planning conflict-free paths for multiple mobile robots in applications such as warehouse logistics, surveillance, construction site  routing, and hazardous  material  transportation.
Existing methods for multi-agent planning typically consider optimizing a single path criteria.
This article develops a novel multi-objective multi-agent planner as well as its variants that are guaranteed to find all Pareto-optimal solutions for the problem.
We also provide an illustrative example of the algorithm to plan paths for multiple agents that transport materials in a construction site while optimizing both path length and risk. In this example, computing and visualizing a set of Pareto-optimal solutions makes it intuitive for the practitioner to understand the underlying trade-off between conflicting objectives and to choose the most preferred solution for execution based on their domain knowledge.
\end{abstract}



\graphicspath{{./figures/}}

\section{Introduction}\label{sec:intro}

\begin{figure}[t]
	\centering
	\includegraphics[width=\linewidth]{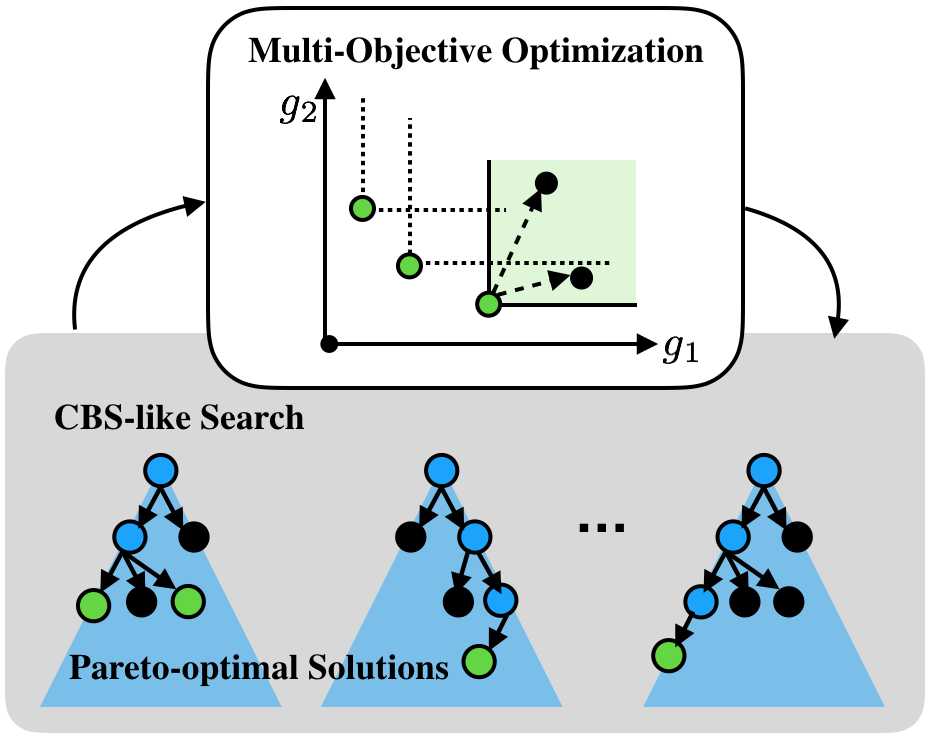}
	\caption{A conceptual visualization of Multi-Objective Conflict-Based Search.
	It leverages Conflict-Based Search (CBS) to resolve conflicts between agents (the lower half of the figure) and compares candidate solutions using the dominance principle (the upper half of the figure) from the Multi-Objective Optimization literature in order to find all conflict-free Pareto-optimal solutions.}
	\label{fig:token_fig}
\end{figure}

\IEEEPARstart{M}{ulti}-Agent Path Finding (MAPF) computes a set of collision-free paths for multiple agents connecting their respective start and goal locations while optimizing a scalar measure of paths. Variants of MAPF have been widely studied in the robotics community over the last few years~\cite{stern2019multi}. In this article, we investigate a natural generalization of the MAPF to include multiple objectives for multiple agents and hence the name {\it Multi-Objective Multi-Agent Path Finding} (MOMAPF). In MOMAPF, agents have to trade-off multiple objectives such as completion time, travel risk and other domain-specific measures. MOMAPF is a generalization of MAPF, and is therefore NP-Hard~\cite{yu2013structure_nphard}.

In the presence of multiple conflicting objectives, in general, no (single) solution can simultaneously optimize all the objectives.
Therefore, the goal of MOMAPF is to find the set of all Pareto-optimal solutions rather than a single optimal solution as in MAPF.
A solution is Pareto-optimal if there exists no other solution that will yield an improvement in one objective without causing a deterioration in at least one of the other objectives.
Finding this set of solutions while ensuring collision-free paths for agents in each solution is quite challenging: even though there are many single-agent multi-objective search algorithms \cite{moastar,mandow2008multiobjective,ulloa2020simple} that can compute all Pareto-optimal solutions, a naive application of such algorithms to the MOMAPF problem may prove to be inefficient as the size of search space grows exponentially with respect to the number of agents~\cite{yu2013structure_nphard,hansen1980bicriterion}. 
Among the algorithms that optimally solve the single-objective MAPF problems, Conflict-Based Search (CBS) \cite{sharon2015conflict} has received significant attention due to its computational efficiency on average. This method has also been extended to solve several other variants of MAPF as noted in \cite{Ma2017Lifelong,cohen2019optimal}. However, how to leverage CBS to solve MOMAPF remains an under-explored question. This article aims to address this gap. By building on multi-objective dominance techniques \cite{ehrgott2005multicriteria,mandow2008multiobjective}, we develop a new algorithm named Multi-Objective Conflict-Based Search (MO-CBS) (Fig.\ref{fig:token_fig}) that is able to compute the entire Pareto-optimal set of collision-free paths with respect to multiple objectives.

MO-CBS takes a similar strategy as CBS to resolve conflicts along paths of agents while extending CBS to handle multiple objectives. MO-CBS begins by computing individual Pareto-optimal paths for each agent ignoring agent-agent conflicts and letting agents follow those paths. When a conflict between agents is found along their paths, MO-CBS splits the conflict by adding constraints to the individual search space of each agent (involved in the conflict) and invokes a single-agent multi-objective planner to compute new individual Pareto-optimal paths subject to those added constraints. In addition, MO-CBS uses dominance rules to select candidate solutions for conflict-checking and compares them until all the candidates are either pruned or identified as Pareto-optimal.

MO-CBS is a search framework in a sense that different (single-agent) planners can be used for the low-level search.
This work investigates using both BOA*~\cite{ulloa2020simple}, a state-of-the-art single-agent bi-objective planner, and NAMOA*-dr~\cite{pulido2015dimensionality}, a single-agent planner for multiple objectives, within the MO-CBS framework.
Additionally, we also develop a variant of MO-CBS that takes a different expansion strategy on its high-level search to improve memory usage.
Compared with an existing approach MOM*~\cite{ren2021subdimensional} that is guaranteed to find all Pareto-optimal solutions for MOMAPF, the numerical results show that the proposed MO-CBS and its variants outperform MOM* in terms of success rates under bounded time in various maps.
Our C++ implementation is available online.\footnote{\url{https://github.com/wonderren/public_cppmomapf}}

Preliminary versions of this research have previously appeared in~\cite{ren2021multi}.
This article contains a new proof of completeness and  optimality of the proposed approach which applies to all the variants of MO-CBS.
We also conduct a new, comprehensive set of experiments to compare MO-CBS with MOM*, and to analyze the performance of MO-CBS variants.
For the rest of this article, we review related work in Sec.~\ref{sec:related} and formulate the problem in Sec.~\ref{sec:problem}. We first revisit CBS in Sec.~\ref{sec:cbs} and then present the basic version of MO-CBS in Sec.~\ref{sec:mocbs}. Variants of MO-CBS are then presented in Sec.~\ref{sec:variants}. We analyze the properties of MO-CBS in Sec.~\ref{sec:analysis} and show numerical results in Sec.~\ref{sec:result}. Finally, conclusion and future work are presented in Sec.~\ref{sec:conclude}.

\section{Related Work}\label{sec:related}

\subsection{Multi-Objective Path Planning}
Multi-objective ({\it single-agent}) path planning (MOPP) problems aim to find a set of Pareto-optimal paths for the agent between its start location and destination with respect to \emph{multiple} objectives.
MOPP arises in applications such as construction site routing~\cite{soltani2004fuzzy}, hazardous material transportation~\cite{erkut2007hazardous}, and others~\cite{montoya2013multiobjective,xu2021multi}.
One common approach to solve a MOPP is to weight the multiple objectives and transform it to a single-objective problem~\cite{emmerich2018tutorial,ehrgott2005multicriteria}.
The transformed problem can then be solved using a corresponding single-objective algorithm.
This approach has two main drawbacks: First, the choice of the weights for the objectives must be known a-priori and requires in-depth domain knowledge which may not always be possible; Second, it may also require one to repeatedly solve the transformed single-objective problem for different sets of weights in order to capture the Pareto-optimal set which is quite challenging \cite{marler2004survey}.

Additionally, MOPP and its variants have been solved directly via graph search techniques~\cite{moastar,mandow2008multiobjective,ulloa2020simple,ren21mosipp,ren2022mopbd} and evolutionary algorithms~\cite{weise2020momapf} where a Pareto-optimal set of solutions is computed exactly or approximately.
These graph-based approaches provide guarantees about finding all Pareto-optimal solutions but can run slow for hard cases, where the number of Pareto-optimal solutions is large. 
MO-CBS developed in this work belongs to this category of search techniques that directly computes a Pareto-optimal set with quality guarantees.

\subsection{Multi-Agent Path Finding}
Various methods have been developed to compute an optimal solution for MAPF problems including A*-based approaches~\cite{standley2010finding,goldenberg2014enhanced}, subdimensional expansion~\cite{wagner2015subdimensional}, compilation-based solver~\cite{surynek2016efficient}, integer programming-based methods~\cite{lam2019bcp} and Conflict-Based Search (CBS)~\cite{sharon2015conflict}.
In addition, different variants of MAPF have also been considered, such as agents moving with different speeds~\cite{andreychuk2022multi,ren2021loosely}, visiting multiple target locations along the path~\cite{ren2021ms,ren22cbss}, pickup-and-delivery tasks~\cite{ma2019lifelong,ma2016optimal}, satisfying kinodynamic constraints~\cite{cohen2019optimal}.
However, all these methods optimize a single objective.

For MOMAPF, heuristic approaches and evolutionary algorithms~\cite{weise2020momapf,weise2021scalable,mai2020modeling,saha2021uavs} have been leveraged to solve variants of MOMAPF.
For example in~\cite{weise2020momapf}, agents are not allowed to wait in place and collisions between the agent's paths are modeled in one of the objectives and not as a constraint.
Recently, in our prior work, by leveraging M*~\cite{wagner2015subdimensional}, we developed Multi-Objective M* (MOM*) \cite{ren2021subdimensional} to solve the MOMAPF with solution quality guarantees.
Similarly to M*, MOM* begins by planning for each agent independently and couples agents for planning by searching in their joint configuration space only when two agents are in conflict with each other.
In addition, MOM* also leverages the dominance principle from the multi-objective optimization literature to compare two partial solutions in order to find all Pareto-optimal solutions.
In this work, we compare the proposed MO-CBS with MOM* in various maps, and the result shows that MO-CBS achieves higher success rates within a runtime limit than MOM* in several maps.

\section{Problem Formulation}\label{sec:problem}

Let index set $I=\{1,2,\dots,N\}$ denote a set of $N$ agents. 
All agents move in a workspace represented as a finite graph $G=(V,E)$, where the vertex set $V$ represents all possible locations of agents and the edge set $E \subseteq V \times V$ denotes the set of all the possible actions that can move an agent between a pair of vertices in $V$.
An edge between two vertices $u, v \in V$ is denoted as $(u,v)\in E$ and the cost of an edge $e \in E$ is a $M$-dimensional positive vector: cost$(e) \in (0,\infty)^M$ with $M$ being a positive integer and each component in cost$(e)$ being a finite number.

In this work, we use a superscript $i,j \in I$ over a variable to represent the specific agent that the variable belongs to (e.g. $v^i\in V$ means a vertex with respect to agent $i$). 
Let $\pi^i(v^i_{1}, v^i_{\ell})$ be a path that connects vertices $v^i_{1}$ and $v^i_{\ell}$ via a sequence of vertices $(v^i_{1},v^i_{{2}},\dots,v^i_{\ell})$ in the graph $G$. 
Let $g^i(\pi^i(v^i_{1}, v^i_\ell))$ denote the $M$-dimensional cost vector associated with the path, which is the sum of the cost vectors of all the edges present in the path, $i.e.$, $g^i(\pi^i(v^i_{1}, v^i_{\ell})) = \Sigma_{j=1,2,\dots,{\ell-1}} \text{cost}(v^i_{{j}}, v^i_{{j+1}})$. 
 
All agents share a global clock and all the agents start their paths at time $t=0$. Each action, either wait or move, for any agent requires one unit of time. 
Any two agents $i,j \in I$ are said to be in conflict if one of the following two cases happens. The first case is a ``vertex conflict'' where two agents occupy the same location at the same time. The second case is an ``edge conflict'' where two agents move through the same edge from opposite directions between times $t$ and $t+1$ for some $t$.   

Let $v_o^i, v^i_f \in V$ respectively denote the initial location and the destination of agent $i$.
Without loss of generality, to simplify the notations, we also refer to a path $\pi^i(v^i_{o}, v^i_{f})$ for agent $i$ between its initial and final locations as simply $\pi^i$. Let $\pi=(\pi^1,\pi^2,\dots, \pi^N)$ represent a joint path for all the agents, which is also called a solution. The cost vector of this solution is defined as the vector sum of the individual path costs over all the agents, $i.e.$, $g(\pi) = \Sigma_i g^i(\pi^i)$.

To compare any two solutions, we compare the cost vectors corresponding to them.
Given two vectors $a$ and $b$, $a$ is said to \textit{dominate} $b$ if every component in $a$ is no larger than the corresponding component in $b$ and there exists at least one component in $a$ that is strictly less than the corresponding component in $b$.
Formally, it is defined as:
\begin{definition}[Dominance \cite{mandow2008multiobjective}]
	Given two vectors $a$ and $b$ of length $M$, $a$ dominates $b$, notationally $a \succeq b$, if and only if $a(m) \leq b(m)$, $\forall m \in \{1,2,\dots,M\}$ and $a(m) < b(m)$, $\exists m \in \{1,2,\dots,M\}$.\footnote{The definition is also referred to as ``Pareto Dominance'' in the literature (e.g.~\cite{emmerich2018tutorial}). To simplify presentation, we call it ``dominance'' in this work.}
\end{definition}
Any two solutions are non-dominated with respect to each other if the corresponding cost vectors do not dominate each other.
A solution $\pi$ is non-dominated with respect to a set of solutions $\Pi$, if $\pi$ is not dominated by any $\pi' \in \Pi$.
Among all conflict-free ($i.e.$ feasible) solutions, the set of all non-dominated solutions is called the {\it Pareto-optimal} set.
In this work, we aim to find all \emph{cost-unique} Pareto-optimal solutions, $i.e.$ any maximal subset of the Pareto-optimal set, where any two solutions in this subset do not have the same cost vector.

\section{A Brief Review of Conflict-based Search}\label{sec:cbs}


\subsection{Conflicts and Constraints}
Let $(i,j,v^i,v^j,t)$ denote a {\it conflict} between agent $i,j \in I$, with $v^i,v^j \in V$ representing the vertex of agent $i,j$ at time $t$. In addition, to represent a vertex conflict, $v^i$ is required to be the same as $v^j$ and they both represent the location where vertex conflict happens. To represent an edge conflict, $v^i, v^j$ denote the adjacent vertices that agent $i,j$ swap at time $t$ and $t+1$.
Given a pair of individual paths $\pi^i, \pi^j$ of agent $i,j \in I$, to detect a conflict, let $\Psi(\pi^i,\pi^j)$ represent a conflict checking function that returns either an empty set if there is no conflict, or the first conflict detected along $\pi^i, \pi^j$.

A conflict $(i,j,v^i,v^j,t)$ can be avoided by adding a corresponding constraint to the path of either agent $i$ or agent $j$. Specifically,
let $\omega^i = (i,u^i_a,u^i_b,t), u^i_a,u^i_b \in V$ denote a {\it constraint} belonging to agent $i$, which is generated from conflict $(i,j,v^i,v^j,t)$ with $u^i_a=v^i, u^i_b=v^j$ and the following specifications:
\begin{itemize}
	\item If $u^i_a = u^i_b$, $\omega^i$ forbids agent $i$ from entering $u^i_a$ at time $t$ and is named as a {\it vertex constraint} as it corresponds to a vertex conflict.
	\item If $u^i_a \neq u^i_b$, $\omega^i$ forbids agent $i$ from moving from $u^i_a$ to $u^i_b$ between time $t$ and $t+1$ and is named as an {\it edge constraint} as it corresponds to an edge conflict.
\end{itemize}
Given a set of constraints $\Omega$, let $\Omega^i \subseteq \Omega$ represent the subset of all constraints in $\Omega$ that belong to agent $i$, and clearly $\Omega = \bigcup_{i\in I} \Omega^i$. Additionally, a path $\pi^i$ is said to be {\it consistent} with respect to $\Omega$ if $\pi^i$ satisfies every constraint in $\Omega^i$. A joint path $\pi$ is consistent with respect to $\Omega$ if every individual path $\pi^i \in \pi$ is consistent.

\subsection{Two Level Search}
CBS is a two level search algorithm. The low-level search in CBS is a single-agent path planner that plans an optimal ($i.e.$ minimum cost) and consistent path for an agent $i$ with respect to the set of constraints in  $\Omega^i$. If there is no consistent path for agent $i$ given $\Omega^i$, the low-level search reports failure.

For the high-level search, CBS constructs a search tree $\mathcal{T}$ with each tree node $P$ containing:
\begin{itemize}
	\item $\pi=(\pi^1,\pi^2,\dots,\pi^N)$, a joint path that connects the start and destination vertices of agents respectively,
	\item $g$, a scalar cost value associated with $\pi$ and
	\item a set of constraints $\Omega$.
\end{itemize}
The root node $P_o$ of $\mathcal{T}$ has an empty set of constraints $\Omega_o=\emptyset$ and the corresponding joint path $\pi_o$ is constructed by the low-level search for every agent respectively with $\Omega^i_o=\emptyset$.

To ``expand'' a high-level search node $P_k=(\pi_k, g_k, \Omega_k)$, where subscript $k$ identifies a specific node, conflict checking $\Psi(\pi^i_k,\pi^j_k)$ is computed for any pair of individual paths in $\pi_k$ with $i,j \in I, i\neq j$. If there is no conflict detected, a solution is found and the algorithm terminates.
Otherwise, for the first detected conflict $(i,j,v^i,v^j,t)$, CBS conducts the following procedures to resolve it.
First, CBS \emph{splits} the detected conflict to generate two constraints $\omega^i=(i,u^i_a=v^i,u^i_b=v^j,t)$ and $\omega^j=(j,u^j_a=v^j,u^j_b=v^i,t)$.
Second, CBS generates two corresponding nodes $P_{l^i}=(\pi_{l^i}, g_{l^i}, \Omega_{l^i}), P_{l^j}=(\pi_{l^j}, g_{l^j}, \Omega_{l^j}) $, where $ \Omega_{l^i} =  \Omega_{k} \cup \{w^i\} $ and $ \Omega_{l^j} =  \Omega_{k} \cup \{w^j\} $.
Finally, CBS lets $\pi_{l^i}\gets\pi_k$ (and $\pi_{l^j}\gets\pi_k$) and then updates the individual path $\pi^i$ in $\pi_{l^i}$ (and $\pi^j$ in $\pi_{l^j}$) by calling the low-level search for agent $i$ (and $j$) with the set of constraints $\Omega_{l^i}$ (and $\Omega_{l^j}$ respectively). If the low-level search fails to find a consistent path for $i$ (or $j$), node $P_{l^i}$ (or $P_{l^j}$) is discarded.

After conflict resolving, CBS inserts generated nodes into OPEN, which is a priority queue containing all candidate high-level nodes. 
CBS solves a (single-objective) MAPF problem to optimality by iteratively selecting candidate node from OPEN with the smallest $g$ cost, detect conflicts, and then either claims success (if not conflict detected) or resolves the detected conflict which generates new candidate nodes.

Intuitively, from the perspective of the search tree $\mathcal{T}$ constructed by CBS, OPEN contains all leaf nodes in $\mathcal{T}$. In each iteration of the high-level search, a leaf node $P_k$ is selected and checked for conflict. CBS either claims success if paths in $P_k$ are conflict-free or generates new leaf nodes.

\section{Multi-objective Conflict-based Search}\label{sec:mocbs}

The proposed Multi-Objective Conflict-Based Search (MO-CBS) is described in Alg.~\ref{alg:mocbs} and visualized in Fig.~\ref{fig:mocbs}.
MO-CBS generalizes CBS and has the following {\it key features} to handle multiple objectives.

\begin{algorithm}[tb]
	\caption{Pseudocode for MO-CBS, {\color{blue}MO-CBS-t}}\label{alg:mocbs}
	\begin{algorithmic}[1]
		\State{\textit{Initialization()}}
		\State{$\mathcal{C}$ $\gets \emptyset$}
		\While{OPEN not empty} 
		\State{$P_k=(\pi_k,\vec{g}_k,\Omega_k) \gets$ OPEN.pop() }
		\State{\color{blue}{// $P_k=(\pi_k,\vec{g}_k,\Omega_k) \gets$ OPEN.pop-tree-by-tree() } }
		\State{\textbf{if} \textit{Filter}($P_k$) \textbf{then continue}}\Comment{End of iteration}
		\State{\textbf{if} no conflict detected in $\pi_k$ \textbf{then}}
		\State{\indent {\it Update}($P_k$)} 
		\State{\indent add $\vec{g}_k$ to $\mathcal{C}$} 
		\State{\indent \textbf{continue} }\Comment{End of iteration}
		\State{$\Omega \gets $ split detected conflict}
		\ForAll{$\omega^i \in \Omega$}
		\State{$\Omega_l = \Omega_k \cup \{\omega^i\}$}
		\State{$\Pi^i_* \gets$ \textit{LowLevelSearch}($i$, $\Omega_l$) }
		\ForAll{$\pi^i_* \in \Pi^i_*$}
		\State{$\pi_l \gets \pi_k$}
		\State{Replace $\pi^i_l$ (in $\pi_l$) with $\pi^i_*$}
		\State{$\vec{g}_l \gets$ compute path cost $\pi_l$}
		\State{$P_l=(\pi_l,\vec{g}_l,\Omega_l)$}
		\State{\textbf{if} not \textit{Filter}($P_l$) \textbf{then}}
		\State{\indent add $P_l$ to OPEN}
		\EndFor
		\EndFor
		\EndWhile \label{}
		\State{\textbf{return} $\mathcal{C}$}
	\end{algorithmic}
\end{algorithm}

\begin{figure*}[htbp]
	\centering
	\includegraphics[width=\linewidth]{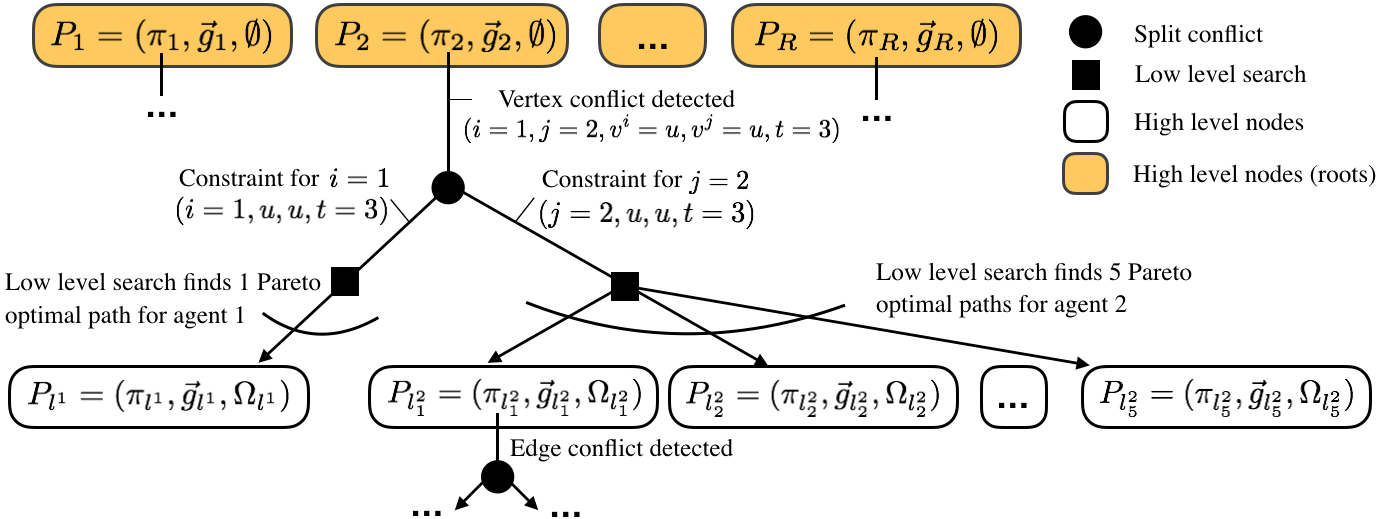}
	\caption{An illustration of the search process of MO-CBS.
	MO-CBS initializes multiple root nodes. MO-CBS iteratively selects a candidate high-level node with a non-dominated cost vector from OPEN, splits the detected conflict to generate constraints, conducts low-level search subject to those constraints, and generates new high-level nodes.
	}
	\label{fig:mocbs}
\end{figure*}

\subsection{Initialization}\label{sec:init}

In MO-CBS, to initialize OPEN (line 1 in Alg.~\ref{alg:mocbs}), a single-agent multi-objective planner (such as NAMOA*-dr~\cite{pulido2015dimensionality}) is used for each agent $i \in I$ separately to compute all cost-unique Pareto-optimal paths, $\Pi^i_o$, for agent $i$. 
A set of joint paths $\Pi_o$ is generated by taking the combination of $\Pi_o^i, \forall i\in I$, 
$i.e.$ $\Pi_o=\{\pi_o | \pi_o=(\pi^1_o,\pi^2_o,\dots,\pi^N_o), \pi^i_o \in \Pi^i_o, \forall i \in I \}$. Clearly, the size of $\Pi_o$ is $|\Pi_o| = |\Pi^1_o|\times|\Pi^2_o|\times\dots\times|\Pi^N_o|$. For each $\pi_o \in \Pi_o$, a corresponding high-level node containing (i) $\pi_o$, (ii) the cost vector associated with $\pi_o$ and (iii) an empty constraint set, is generated and added into OPEN.
Intuitively, while the original CBS initializes a single root node and a single search tree $\mathcal{T}$, MO-CBS initializes a number of $R=|\Pi_o|$ root nodes and a ``search forest'' $\mathcal{T}_r, r \in \{1,2,\dots,R\}$ where each tree $\mathcal{T}_r$ corresponds to a root node.\footnote{The idea of using a CBS-like search forest to solve multi-agent path planning problems have also been investigated in~\cite{honig2018conflict,greshler2021cooperative}.}

In this work, let $\mathcal{C}^*$ denote the Pareto-optimal front of the given problem instance: the set of cost vectors corresponding to the Pareto-optimal set.
Let $\mathcal{C}$ denote a set of cost vectors, where each cost vector corresponds to a conflict-free joint path ($i.e.$ solution) that is found during the search.
$\mathcal{C}$ is initialized to be an empty set (line 2).

\subsection{Finding a Solution}

For every search iteration in MO-CBS (lines 3-21), a high-level node $P_k$ with non-dominated cost vectors among all nodes in OPEN is popped.\footnote{In practice, the lexicographic order of cost vectors is often used to prioritize nodes in OPEN~\cite{ulloa2020simple,pulido2015dimensionality,ren2022mopbd} and it can guarantee that every popped node has a non-dominated cost vector among all nodes in OPEN. This work follows this common practice. Other types of prioritization for the candidates in OPEN can also be used within the MO-CBS framework.}
The popped node $P_k=(\pi_k,\vec{g}_k,\Omega_k)$ is first checked for dominance in procedure \textit{Filter} (line 6), where $\vec{g}_k$ is compared with each cost vector in $\mathcal{C}$.
If there exists a vector $\vec{g}\in\mathcal{C}$ such that $\vec{g}\leq\vec{g}_k$ ($i.e.$ every component in $\vec{g}$ is no larger than the corresponding component in $\vec{g}_k$),\footnote{Note that $\vec{g}\leq\vec{g}_k$ is equivalent to (i) $\vec{g}$ is not dominated by $\vec{g}_k$ and (ii) $\vec{g}$ is also not equal to $\vec{g}_k$.} then node $P_k$ cannot lead to a cost-unique Pareto-optimal solution and is thus discarded ($i.e.$ filtered), and the current search iteration ends.
With the \textit{Filter} procedure, each vector in set $\mathcal{C}$ is guaranteed to be unique.

If $\pi_k$ is conflict-free (lines 7-10), a solution node ($i.e.$ a high-level node containing a solution) is identified, and the cost vector $\vec{g}_k$ is first used to update $\mathcal{C}$ in procedure \textit{Update} and then added to $\mathcal{C}$.
The purpose of \textit{Update} is to ensure that an existing cost vector in $\mathcal{C}$ is removed if it is dominated by $\vec{g}_k$.
Specifically, \textit{Update}($P_k$) uses the cost vector $\vec{g}_k$ in $P_k$ to compare with all existing solution cost vectors (that have already been found during the search) in $\mathcal{C}$, and if $\vec{g}_k \succeq \vec{g}, \vec{g}\in \mathcal{C}$, then $\vec{g}$ is removed from $\mathcal{C}$. (Note that $\vec{g}_k$ cannot be equal to $\vec{g}$, since $\vec{g}_k$ would have been discarded in \textit{Filter} otherwise.)
This \textit{Update} procedure is necessary due to the fact that a search forest $\mathcal{T}_r, r \in \{1,2,\dots,R\}$ (rather than a single search tree) is constructed by MO-CBS. When a solution is found, it is not guaranteed to be Pareto-optimal.
When a new solution $\pi_k$ (in high-level node $P_k$) is found, \textit{Update}($P_k$) removes existing solution cost vectors in $\mathcal{C}$ with dominated cost vectors.
As a result, when the algorithm terminates, $\mathcal{C}$ is guaranteed to be the same as the Pareto-optimal front $\mathcal{C}^*$. Additionally, for each $\vec{g} \in \mathcal{C}$, there exists a corresponding solution $\pi$ that is found during the search. Thus, when Alg.~\ref{alg:mocbs} terminates with $\mathcal{C}=\mathcal{C}^*$, all cost-unique Pareto-optimal solution are found.

For readers that are familiar with CBS: while CBS terminates when the first solution is identified, MO-CBS continues to search when a solution is identified and terminates only when OPEN is empty in order to identify all cost-unique Pareto-optimal solutions. 

\subsection{Conflict Resolution}\label{sec:conflict_reso}
When a node $P_k=(\pi_k,\vec{g}_k,\Omega_{k})$ is popped from OPEN, if $\pi_k$ contains a conflict, just as in CBS, the detected conflict is split into two constraints and a new set of constraints $\Omega_l$ is generated correspondingly. Given an agent $i$ and a constraint set $\Omega_{l}$, the low-level search (which is explained next) is invoked to compute individual Pareto-optimal paths that are consistent with respect to $\Omega_{l}$ for agent $i$.

Given $\Omega_{l}$ and an agent $i$, while in CBS, only one individual optimal path for agent $i$ (that is consistent with $\Omega_{l}$) is computed, in MO-CBS, there can be multiple consistent Pareto-optimal individual paths for agent $i$.
To find all of them, the low-level search employs a single-agent multi-objective planner (Sec.~\ref{mocbs:sec:low_level}) to search a time-augmented graph $G^{t} = (V^{t},E^{t}) = G \times \{0,1,\dots,T\}$, where each vertex in $v \in V^{t}$ is defined as $v=(u,t), u\in V, t \in \{0,1,\dots,T\}$ and $T$ is a pre-defined time horizon (a large positive integer). Edges within $G^{t}$ is represented as $E^{t}= V^{t} \times V^{t}$ where $(u_1,t_1),(u_2,t_2)$ is connected in $G^t$ if $(u_1,u_2) \in E$ and $t_2=t_1 + 1$. Wait in place is also allowed in $G^{t}$ ($i.e.$ $(u,t_1),(u,t_1+1), u \in V$ is connected in $G^t$). In addition, all vertices and edges in $G^t$ that correspond to vertex constraints and edge constraints in $\Omega^i_l \subseteq \Omega_l$ are removed from the time augmented graph $G^t$.

The low-level search guarantees to return a set of consistent Pareto-optimal individual paths $\Pi^i_*$ for agent $i$ subject to the given constraint set.
For each path $\pi^i_* \in \Pi^i_*$ computed by the low-level search, a corresponding joint path $\pi_l$ is generated by first making a copy of $\pi_k$ and then update the individual path $\pi^i_l$ in $\pi_l$ with $\pi^i_*$ (lines 16-17).
If the cost vector of $\pi_l$ is neither dominated by nor equal to the cost vector of any solution cost vector in $\mathcal{C}$ (line 20), a new node $P_l=(\pi_l,\vec{g}_l,\Omega_l)$ is generated and inserted into OPEN.



\subsection{Relationship to CBS}
With only one objective ($i.e.$ $M=1$), MO-CBS is equivalent to CBS in the following sense.
Dominance between vectors becomes the ``less than'' comparison between scalars and the candidate with the minimum $g$ cost in OPEN is popped in each iteration. When the first solution with the minimum cost $g^*$ is found, all other nodes in OPEN must have a cost value no less than $g^*$, and are thus discarded by the \textit{Filter} procedure, which makes OPEN empty and leads to the termination of MO-CBS.
Additionally, the low-level search returns an individual optimal path for an agent when invoked. Only one root node is generated at the initialization step and there is one corresponding search tree built during the search.

\section{Variants of MO-CBS}\label{sec:variants}

\subsection{Tree-By-Tree Expansion for the High-Level Search}\label{sec:mocbs-t}

In MO-CBS, a node with a non-dominated cost vector is selected from OPEN and expanded (conflict checking and splitting). This expansion strategy has two drawbacks.
First, all root nodes need to be generated so that a non-dominated one can be selected.
Considering an example with ten agents and each agent has ten individual Pareto-optimal paths, MO-CBS needs to generate $10^{10}$ root nodes, which is computationally prohibitive.
The second drawback is that nodes are selected in a ``breadth-first'' manner in a sense that the selected nodes can belong to different trees.
As the number of agents (or objectives) increases, this expansion strategy may lead to a large number of expansions before finding the first solution.

Here we propose a new expansion strategy to bypass these limitations. Let candidates in OPEN be sorted by the tree $\mathcal{T}_r$ they belong to, and let OPEN$_r$ denote the open list that contains only candidate nodes in tree $\mathcal{T}_r, r\in \{1,2,\dots,R\}$.
Clearly, OPEN$=\bigcup_{r\in{1,2,\dots,R}}$OPEN$_r$.
Instead of selecting a non-dominated node in OPEN as MO-CBS does, here, only nodes in OPEN$_1$ are considered for selection at first.
The selected node is then expanded in the same manner as MO-CBS does.
As any newly generated nodes belong to $\mathcal{T}_1$, these nodes must be inserted into OPEN$_1$.
Only when OPEN$_1$ depletes, the algorithm then selects candidates from OPEN$_2$ (and then OPEN$_3$, and so on) for expansion.
The algorithm terminates when OPEN$_R$ is depleted. We denote MO-CBS with such a ``tree-by-tree'' (abbreviated as ``-t'') node selection strategy as MO-CBS-t.
MO-CBS-t enables \emph{on-demand} generation of roots and performs a ``depth-first'' like search by exhaustively examining one tree after another.
This allows MO-CBS-t to start the search without initializing all the roots (which is verified in Sec.~\ref{mocbs:sec:result:high-level}).

\subsection{Different Low-Level Planners}\label{mocbs:sec:low_level}
MO-CBS is a search framework in a sense that different low-level planners can be used within the framework, as long as the low-level planner can find all individual Pareto-optimal paths in a time-augmented graph $G^t$ as described in Sec.~\ref{sec:conflict_reso}.

Among the existing single-agent multi-objective search algorithms, NAMOA*~\cite{mandow2008multiobjective} is a popular A*-like multi-objective planner.
NAMOA*-dr~\cite{pulido2015dimensionality} is an improved version of NAMOA* with the so-called ``dimensionality reduction'' (-dr) technique.
Both NAMOA* and NAMOA*-dr can handle an arbitrary number of objectives.
Recently, NAMOA*-dr is further expedited by BOA*~\cite{ulloa2020simple} when there are only two objectives.
We refer the reader to the BOA* paper~\cite{ulloa2020simple} for a detailed discussion about the technical difference between those algorithms.

All these algorithms can be applied to search the aforementioned time-augmented graph $G^t$ and be used as the low-level planner of MO-CBS.
We use notation BOA*-st and NAMOA*-dr-st (-st stands for space-time) to indicate that the planner is applied to the time-augmented graph.
In this work, to handle an arbitrary number of objectives, we use NAMOA*-dr-st as the low-level planner of MO-CBS and denote the corresponding algorithm MO-CBS-n.
When there are only two objectives, we use BOA*-st as the low-level planner of MO-CBS and denote the corresponding algorithm MO-CBS-b.
If the aforementioned tree-by-tree expansion strategy is used, we add ``-t'' to denote the corresponding variant (e.g. MO-CBS-tb, MO-CBS-tn).

\section{Analysis}\label{sec:analysis}

\subsection{Pareto-optimality}
Let $\Pi_*$ denote the set of all Pareto-optimal (solution) joint paths for a given MOMAPF problem instance.
Note that for two solutions $\pi, \pi' \in \Pi_*$, it's possible that their cost vectors $\vec{g}(\pi),\vec{g}(\pi')$ are the same.
At any time of the search, define $\Pi_* |\mathcal{C} := \{\pi : \pi \in \Pi_*, \vec{g}(\pi) \notin \mathcal{C}\}$. Intuitively, $\Pi_* |\mathcal{C}$ is the subset of $\Pi_*$ whose cost vectors have not yet been included in $\mathcal{C}$ during the search.
Additionally, ``expanding'' a high-level node means checking for conflicts and splitting the detected conflict as aforementioned in Sec.~\ref{sec:mocbs}.

\vspace{2mm}
\begin{definition}[$CV set$]\label{def:CV-set}
	Given a high-level node $P=(\pi,\vec{g},\Omega)$, let $CV(P)$ be the set of all joint paths that are (i) consistent with $\Omega$, and (ii) conflict-free ($i.e.$ valid).
\end{definition}

Correspondingly, if $\pi' \in CV(P)$, we say $P$ \emph{permits} $\pi'$.
Intuitively, each joint path in $CV(P)$ is a (conflict-free) solution joint path that satisfies all constraints in $\Omega$.

\vspace{2mm}

\begin{definition}\label{def:permit_node}
For each $\pi_* \in \Pi_* | \mathcal{C}$, let $\mathcal{P}(\pi_*)$ denote a high-level search node $(\pi,\vec{g},\Omega)$ such that (i) $\mathcal{P}(\pi_*)$ permits $\pi_*$ and (ii) for each agent $i\in I$, $\vec{g}(\pi^i) \leq \vec{g}(\pi_*^i)$.
\end{definition}

Correspondingly, if a node $P_k$ satisfies Def.~\ref{def:permit_node} for some $\pi_* \in \Pi_* | \mathcal{C}$, we say $P_k$ is a $\mathcal{P}$-node of $\pi_*$.

\vspace{2mm}
\begin{corollary}\label{corollary:leq_cost}
For a $\pi_* \in \Pi_* | \mathcal{C}$ and a corresponding $\mathcal{P}$-node of $\pi_*$, which is denoted as $(\pi,\vec{g},\Omega)$, we have $\vec{g}(\pi) \leq \vec{g}(\pi_*)$.
\end{corollary}

\vspace{2mm}
\begin{lemma}\label{lem:never_filtered}
	During the search iterations, for any $\pi_* \in \Pi_*|\mathcal{C}$, if $\mathcal{P}(\pi_*)$ exists and is popped from OPEN for expansion, $\mathcal{P}(\pi_*)$ will not be filtered in the procedure \textit{Filter}.
\end{lemma}

\begin{proof}
    We prove this Lemma by contradiction.
    By Def.~\ref{def:permit_node}, node $\mathcal{P}(\pi_*)$ has a cost vector $\vec{g}\leq \vec{g}(\pi_*)$.
    If $\mathcal{P}(\pi_*)$ is removed by the procedure \textit{Filter}, there must exist a feasible solution $\pi'$  with cost vector $\vec{g}(\pi') \in \mathcal{C}$ such that $\vec{g}(\pi') \leq \vec{g}$. 
    Hence, we have $\vec{g}(\pi') \leq \vec{g}(\pi_*)$. This implies $\vec{g}(\pi') = \vec{g}(\pi_*)$ because $\pi$ and $\pi_*$ are feasible solutions, and $\pi_*$ is Pareto-optimal. However, $\vec{g}(\pi') = \vec{g}(\pi_*)$ is not possible because if $\vec{g}(\pi')\in \mathcal{C}$, then by definition, $\pi_* \notin \Pi_*|\mathcal{C}$. Hence proved.
    \hfill$\blacksquare$
\end{proof}

\vspace{2mm}
\begin{lemma}\label{lem:p-node-still-exists}
    Let a $\pi_* \in \Pi_* | \mathcal{C}$ be such that there exists a $\mathcal{P}$-node of $\pi_*$ (denoted as $P_k=(\pi_k,\vec{g}_k,\Omega_k)$) in OPEN, and let $\pi_k$ have conflicts. Then, if $P_k$ is popped from OPEN and expanded (lines 11-21), there still exists a $\mathcal{P}$-node of $\pi_*$ after expansion in OPEN.
\end{lemma}

\begin{proof}
    If $\pi_k$ has a conflict between agent $i,j \in I$ (line 11), during the expansion, MO-CBS splits the conflict, generates two constraints $\omega_i$, $\omega_j$ and invokes the low-level planner (line 14) to find individual cost-unique Pareto-optimal paths subject to the new set of constraints $\Omega_k\bigcup \{\omega_i\}$ or $\Omega_k\bigcup \{\omega_j\}$, which results in two sets $\{P_{l^i}\},\{P_{l^j}\}$ of new high-level nodes (line 19).
    Since $\pi_*$ must satisfy at least one of the constraints $\omega_i,\omega_j$, at least one set of nodes $\{P_{l^i}\},\{P_{l^j}\}$ must permit $\pi_*$.
    
    Without losing generality, let $\{P_{l^i}\}$ be a set of nodes that permits $\pi_*$, and let $\Pi_l^i$ denote a set of all individual cost-unique Pareto-optimal path for agent $i$ that is computed by the low-level planner after adding constraint $\omega_i$.
    Note that there is a one-one correspondence between nodes in $\{P_{l^i}\}$ and individual paths in $\Pi_l^i$ (lines 16-19).
    So, there exists at least one individual path $\pi^i_l \in \Pi_l^i$ such that $\vec{g}(\pi^i_l)\leq \vec{g}(\pi^i_*)$ because otherwise $\pi_*^i$ is non-dominated by any solution in $\Pi^i_l$. It means $\pi_*^i$ is a cost-unique Pareto-optimal path that satisfies all constraints in $\Omega_k\bigcup \{\omega_i\}$ and the low-level planner does not find it, which is impossible.
    Let $(\pi_l,\vec{g}_l,\Omega_l) \in \{P_{l^i}\}$ denote the generated high-level node corresponding to $\pi^i_l$, then $\vec{g}(\pi_l^j)=\vec{g}(\pi_k^j)\leq\vec{g}(\pi_*^j), \forall j\in I, j\neq i$ (by lines 16-17 in MO-CBS, and note that $P_k$ is a $\mathcal{P}$-node of $\pi_*$).
    So, there exists a node in $\{P_{l^i}\}$ that is a $\mathcal{P}$-node of $\pi_*$.
    
    Finally, by Lemma~\ref{lem:never_filtered}, $\mathcal{P}$-node of $\pi_*$ cannot be filtered (line 20), and is thus added to OPEN.
    \hfill$\blacksquare$
\end{proof}

\vspace{2mm}
\begin{lemma}\label{lem:permits}
    During any iteration of the algorithm, if $\Pi_* | \mathcal{C}$ is non-empty, then for each $\pi_* \in \Pi_* | \mathcal{C}$, there exists at least one $\mathcal{P}(\pi_*)$ in OPEN.
\end{lemma}

\begin{proof}
    We show this Lemma by mathematical induction.
    
    {\it Base case:} During the initialization step of MO-CBS, all individual Pareto-optimal paths of each agent are computed and all possible combinations are enumerated to generate initial joint paths and root nodes. Each root node has an empty constraint set and permits all $\pi_* \in \Pi_*$. Thus, right after initialization ($i.e.$ after line 1), this Lemma holds.
    
    {\it Assumption:} Assume the Lemma holds at the start of the $k$-th iteration of the {\it while} loop of MO-CBS.
    
    {\it Induction:} During the $k$-th iteration of the {\it while} loop, let $P_k=(\pi_k,\vec{g}_k,\Omega_k)$ denote a node that is popped from OPEN (line 4).
    For any $\pi_* \in \Pi_*|\mathcal{C}$, if $P_k$ is not a $\mathcal{P}$-node of $\pi_*$, then by assumption, there must exist another node $P_k'$ in OPEN, which is a $\mathcal{P}$-node of $\pi_*$. Since $P_k'$ is not popped from OPEN during the $k$-th iteration, $P_k'$ is still in OPEN and the Lemma holds.
    Hence, we only need to consider the case where the popped node is a $\mathcal{P}$-node of $\pi_*$.
    By Lemma~\ref{lem:never_filtered}, $P_k$ is not removed during the filtering step (line 6) of the algorithm. Now, $\pi_k$ must either be conflict-free or have conflicts:
    \begin{itemize}
        \item If $\pi_k$ is conflict-free (line 7), since $\vec{g}(\pi_k)\leq \vec{g}(\pi_*)$, $\pi_k$ must also be Pareto-optimal and $\vec{g}(\pi_k) = \vec{g}(\pi_*)$. Since $\vec{g}(\pi_k)$ is added to $\mathcal{C}$ (line 9), by definition, $\pi_*$ does not belong to $\Pi_* | \mathcal{C}$ any more.
        
        \item If $\pi_k$ has a conflict, as shown in Lemma~\ref{lem:p-node-still-exists}, there is still a $\mathcal{P}$-node in OPEN after the expansion of $P_k$.
    \end{itemize}
    Therefore, at the end of the $k$-th iteration of  MO-CBS, the Lemma holds. Hence, proved.
    \hfill$\blacksquare$
\end{proof}

\vspace{2mm}
\begin{theorem}[Pareto-optimality]\label{thm:optimal}
	For a given problem instance, MO-CBS finds the entire Pareto-optimal front $\mathcal{C}^*$, if it exists.
\end{theorem}

\begin{proof}
     During the search of MO-CBS, by Lemma~\ref{lem:permits}, for each $\pi_* \in \Pi_*$, either $\pi_*$ is permitted by some high-level node $\mathcal{P}(\pi_*)$ in OPEN, or $\vec{g}(\pi_*) \in \mathcal{C}$. Therefore, until a Pareto-optimal solution with a cost vector equal to $\vec{g}(\pi_*)$ is added to $\mathcal{C}$, some high-level node $\mathcal{P}(\pi_*)$ will exist in OPEN. MO-CBS terminates only when OPEN depletes, which means all nodes in OPEN are either filtered or expanded. Therefore, MO-CBS will find the entire Pareto-optimal front.
    \hfill$\blacksquare$
\end{proof}

\subsection{Completeness}

A MOMAPF problem instance is \emph{feasible} if there exists at least one feasible ($i.e.$ conflict-free) joint path for all agents.
A MOMAPF problem instance is infeasible otherwise.
An algorithm is \emph{complete} if:
\begin{itemize}
    \item (Statement-1) The algorithm returns a solution in finite time, if the given problem instance is feasible.
    \item (Statement-2) The algorithm reports failure in finite time, if the given problem instance is infeasible.
\end{itemize}
We first consider (Statement-1).

\vspace{2mm}
\begin{lemma}\label{lem:terminate_finite_time}
	MO-CBS terminates in finite time, if the given MOMAPF problem instance is feasible.
\end{lemma}

\begin{proof}
    $\mathcal{C}^*$ contains a finite set of Pareto-optimal cost vectors. 
    MO-CBS never expands a high-level node $P$ with a cost vector $\vec{g} \geq \vec{g}^*, \exists \vec{g}^* \in \mathcal{C}^*$, ($i.e.$ every component in $\vec{g}$ is no less than the corresponding component in $\vec{g}^*$), since such a node $P$ is removed by the \emph{Filter} procedure.
    Graph $G$ is finite ($i.e.$ has finite number of vertices and edges). Each edge\footnote{Note that wait in place actions are represented as self-loops in the graph, which are also included in the edge set $E$.} in the graph has cost$(e) \in (0,\infty)^M$. Hence, there are only a finite number of joint paths $\pi$ connecting the starts and destinations of all agents such that $\vec{g}(\pi) \ngeq \vec{g}^{*}, \exists \vec{g}^* \in \mathcal{C}^*$.
    In each search iteration, MO-CBS either identifies a feasible solution (a conflict-free joint path), or detects a conflict and generates new constraints which prevents at least one joint path from being generated (lines 12-19) in subsequent search iterations.
    Hence, MO-CBS terminates in finite time.
    \hfill$\blacksquare$
\end{proof}



\vspace{2mm}
\begin{theorem}[Completeness]\label{thm:complete}
	MO-CBS finds a solution in finite time, if the given MOMAPF problem instance is feasible.
\end{theorem}
\begin{proof}
    By Lemma~\ref{lem:terminate_finite_time}, MO-CBS terminates in finite time.
    By Theorem~\ref{thm:optimal}, MO-CBS finds a solution at termination.
    \hfill$\blacksquare$
\end{proof}

\vspace{1mm}
We now discuss (Statement-2).
If the given MOMAPF problem instance is infeasible, then MO-CBS may not terminate.
To overcome this issue, similar to~\cite{sharon2015conflict}, we can run some feasibility checking before running MO-CBS. Specifically, given a MOMAPF instance, a corresponding MAPF instance is generated by assigning each edge in graph $G$ a (scalar) unit cost value. Then the generated MAPF instance is verified in polynomial time with the method in~\cite{yu2015pebble} to check whether this MAPF instance is feasible.
It's obvious that the given MOMAPF instance is feasible if and only if the generated MAPF instance is feasible.

\vspace{1.5mm}
\noindent\textbf{Remark.}
The definition of completeness in this work is the same as the one in~\cite{sharon2015conflict}.
The proofs are applicable to the basic version of MO-CBS as well as its variants.
Specifically, the variant MO-CBS-t differs from MO-CBS in a sense that it re-orders the expansions during the search, and both theorems still hold.
For MO-CBS variants that use different low-level planners, since those low-level planners are all guaranteed to find all individual cost-unique Pareto-optimal paths, Lemma~\ref{lem:p-node-still-exists} is still correct.
Consequently, both theorems still hold.

\section{Numerical Results}\label{sec:result}

\graphicspath{{figures/}}

\subsection{Test Settings, Implementation and Baseline}\label{sec:result:settings}
We implement all four variants MO-CBS-n, MO-CBS-tn, MO-CBS-b, MO-CBS-tb in C++.
We test on a Ubuntu 20.04 laptop with an Intel Core i7-11800H 2.40GHz CPU and 16 GB RAM without multi-threading or compiler optimization.
For comparison, we implement the recent MOM* \cite{ren2021subdimensional} in C++ as a baseline, which can also guarantee finding all cost-unique Pareto-optimal solutions as MO-CBS does.
Another method to compute all cost-unique Pareto-optimal solutions is applying a single-agent multi-objective planner to search the joint graph of all agents. This method has been shown to be computationally inefficient, as the size of the joint graph grows exponentially with respect to the number of agents~\cite{ren2021subdimensional}, which is thus omitted in this article.

In our implementation, for each agent, the heuristic vector is computed by running $M$ exhaustive backwards Dijkstra search from that agent's destination: the $m$-th Dijkstra search ($m=1,2,\dots, M$) uses edge cost values $c_m(e),\forall e\in E$ ($i.e.$ the $m$-th component of the cost vector $\vec{c}(e)$ of all edges).
In our implementation, for the high-level search, all nodes in OPEN are prioritized in the lexicographic order based on their $\vec{g}$-vectors and the minimum one is popped from OPEN in each search iteration.
This implementation guarantees that every popped high-level node has a non-dominated cost vector among all nodes in OPEN.

We select (grid) maps of different types from a MAPF data set \cite{stern2019multi}. For each map, an un-directed graph $G$ is generated by making each grid four-connected.
To assign cost vectors to edges in $G$, we follow the convention in~\cite{pulido2015dimensionality} by assigning each edge an $M$-dimensional cost vector with each component being an integer randomly sampled from $[1,C_{max}]$, where $C_{max}$ takes different values in the following sections.
We use the start-goal pairs from the ``random'' category in the data set \cite{stern2019multi}, and for each map, there are 25 instances.
We set a runtime limit of 300 seconds for each instance.

\subsection{MO-CBS Low-Level Search}\label{mocbs:sec:result:low-level}

We begin by investigating different low-level planners within the framework of MO-CBS.
We fix $M=2$ and compare MO-CBS-b and MO-CBS-n.
These two planners expand nodes in the same order for the high-level search, and the only difference between them is the low-level search.

\subsubsection{Different $C_{max}$}
First, we set $M=2,N=4$ (fixed) and vary $C_{max}$ in an empty $16\times16$ map.
Let $\bar{t}$ denote the average runtime (in micro-seconds) of the low-level planner per call during the MO-CBS search.
As shown in Table~\ref{mocbs:tab:low_level_cmax} (a), the low-level planner of MO-CBS-b ($i.e.$ BOA*-st) runs up to twice as fast as the low-level planner of MO-CBS-n ($i.e.$ NAMOA*-dr-st), which can be observed by comparing 15.2ms against 32.8ms in the $C_{max}=8$ row.
The advantage of BOA*-st over NAMOA*-dr-st is more obvious as $C_{max}$ increases.
We discuss the reason for this in the ensuing paragraphs.
We also show the corresponding number of expansions (\#Exp) of the low-level planners in Table~\ref{mocbs:tab:low_level_cmax} (b).
Note that \#Exp is not an accurate indicator to compare the computational efforts of BOA*-st and NAMOA*-dr-st, since the computational effort of each expansion in BOA*-st is in general cheaper than NAMOA*-dr-st due to the improved dominance checks. More details can be found in~\cite{ulloa2020simple}.

\begin{table}[tb]
	\centering
		\tabcolsep=0.1cm
	\renewcommand{\arraystretch}{1.2}
	\begin{tabular}{ | l | l | l | }
		\hline
		\textbf{(a)}& \multicolumn{2}{c}{min. / median / max. low-Level RT (unit: ms)} \vline
		\\ \hline
		$C_{max}$ & MO-CBS-b & MO-CBS-n
		\\ \hline
		$2$ & {1.9 / 2.9 / 6.2} & {2.0 / 3.8 / 7.9}
		\\ \hline
		$5$ & {2.1 / 5.0 / 17.7} & {2.2 / 6.6 / 26.4}
		\\ \hline
		$8$ & {2.1 / 6.3 / 15.2} & {2.2 / 8.2 / 32.8}
		\\ \hline
		\hline
		\textbf{(b)}& \multicolumn{2}{c}{min. / median / max. low-Level \#Exp.} \vline
		\\ \hline
		$C_{max}$ & MO-CBS-b & MO-CBS-n
		\\ \hline
		$2$ & {11.0 / 43.1 / 157} & {9.8 / 41.7 / 154}
		\\ \hline
		$5$ & {12.3 / 112 / 533} & {10.5 / 107 / 534}
		\\ \hline
		$8$ & {13.8 / 158 / 460} & {12.0 / 151 / 718}
		\\ \hline
	\end{tabular}
	\caption{Runtime (RT) data of the low-level planner of MO-CBS-b ($i.e.$ BOA*) and MO-CBS-n ($i.e.$ NAMOA*-dr) with $M=2,N=4$ (fixed) and varying $C_{max}$ in the empty $16\times16$ map.
	Let $\bar{t}$ denote the average runtime (in microseconds) of the low-level planner per call during the MO-CBS search.
	Table (a) shows the minimum, median and maximum of $\bar{t}$ over all instances.
	Table (b) shows the same statistics of the number of expansions (\#Exp.) of the low-level planner per call.
	BOA*-st runs up to twice as fast as NAMOA*-dr-st, and the advantage of BOA*-st is more obvious as $C_{max}$ increases.}
	\label{mocbs:tab:low_level_cmax}
\end{table}

\subsubsection{Different Maps}
We then show how the size of the map affects the low-level planner.
Table~\ref{mocbs:tab:low_level_maps} (a) shows the minimum, median and maximum $\bar{t}$ over all instances.
As the map size increases, both planners need more runtime per call in general.
In the map den312d of size 65x81, NAMOA*-dr-st takes up to around 3 seconds per call while BOA*-st needs around 2 seconds.
Considering that MO-CBS needs to iteratively invoke the low-level planner, a speed-up in the low-level planner can help with the overall MO-CBS search, which is verified in the resulting success rates of MO-CBS-b and MO-CBS-n in the test with the den312d map: out of the 25 instances, MO-CBS-b succeeds 20 instances while MO-CBS-n succeeds 18 instances.

\begin{table}[tb]
	\centering
	\renewcommand{\arraystretch}{1.2}
	\begin{tabular}{ | l | l | l | }
		\hline
		\textbf{(a)}& \multicolumn{2}{c}{min. / median / max. low-Level RT (unit: ms)} \vline
		\\ \hline
		Map & MO-CBS-b & MO-CBS-n
		\\ \hline
		empty 16x16 & {1.9 / 2.9 / 6.2} & {2.0 / 3.8 / 7.9}
		\\ \hline
		random 32x32 & {4.4 / 9.8 / 57.7} & {4.6 / 11.5 / 78.1}
		\\ \hline
		den312d 65x81 & {22.9 / 142.3 / 2267.4} & {24.7 / 188.7 / 3062.2}
		\\ \hline
		\hline
		\textbf{(b)}& \multicolumn{2}{c}{min. / median / max. low-Level \#Exp.} \vline
		\\ \hline
		Map & MO-CBS-b & MO-CBS-n
		\\ \hline
		empty 16x16  & {11.0 / 43.1 / 157} & {9.8 / 41.7 / 154}
		\\ \hline
		random 32x32 & {10.3 / 197 / 2098} & {9.0 / 195 / 2195}
		\\ \hline
		den312d 65x81 & {281 / 3925 / 68476} & {278 / 3916 / 64232 }
		\\ \hline
	\end{tabular}
	\caption{Similarly to Table~\ref{mocbs:tab:low_level_cmax}, this table reports the runtime (RT) data of the low-level planners when $M=2,N=4,C_{max}=2$ (fixed) in maps of different types and sizes.
	Table (a) and (b) show the statistics of $\bar{t}$ and of \#Exp. respectively over all instances.
	BOA*-st runs faster than NAMOA*-dr-st in all the maps. In the last den312d map, MO-CBS-n solves 18 instances while MO-CBS-b solves 20 ($i.e.$ two more) instances because of the faster low-level planner BOA*-st.}
	\label{mocbs:tab:low_level_maps}
\end{table}

\subsubsection{Discussion and Summary}
Finally, we report the statistics of the number of root nodes (\#Root) of MO-CBS over all instances corresponding to the tests in Table~\ref{mocbs:tab:low_level_cmax} and~\ref{mocbs:tab:low_level_maps}.
Note that \#Root is the product of the numbers of Pareto-optimal individual paths of each agent, and the number of agents is fixed ($N=4$) in this experiment. 
The geometric mean of \#Root over agents is an indicator of the number of individual Pareto-optimal paths for each agent.
From Table~\ref{mocbs:tab:num_root}, as $C_{max}$ increases or the size of the map increases, \#Root grows correspondingly, and it indicates that each agent tends to have more individual Pareto-optimal paths.
Combined with Table~\ref{mocbs:tab:low_level_cmax} and~\ref{mocbs:tab:low_level_maps}, it shows that, finding more Pareto-optimal paths burdens a low-level planner in general.

\begin{table}[tb]
	\centering
	\renewcommand{\arraystretch}{1.2}
	\begin{tabular}{ | l | l | }
		\hline
		 & {min. / median / max. \#Root}
		\\ \hline
		empty 16x16, $C_{max}=2$ & {2 / 12 / 108}
		\\ \hline
		empty 16x16, $C_{max}=5$  & {6 / 150 / 1458}
		\\ \hline
		empty 16x16, $C_{max}=8$  & {12 / 330 / 2205}
		\\ \hline
		random 32x32, $C_{max}=2$  & {2 / 30 / 720}
		\\ \hline
		den312d 65x81 $C_{max}=2$ & {42 / 1920 / 24948}
		\\ \hline
	\end{tabular}
	\caption{The minimum, median and maximum number of roots (\#Root) of the MO-CBS search with $N=4$ (fixed) and varying $C_{max}$ in various maps.
	As $C_{max}$ increases or the size of the map increases, \#Root grows correspondingly, and it indicates that each agent tends to have more individual Pareto-optimal paths, which burdens the low-level planner.}
	\label{mocbs:tab:num_root}
\end{table}

To summarize, first, instances with larger $C_{max}$ and larger maps tend to have more Pareto-optimal individual paths, and it takes the low-level planner more time and expansions to find those Pareto-optimal paths in general.
Second, BOA*-st clearly outperforms NAMOA*-dr-st in terms of the runtime (when $M=2$).
Therefore, for the rest of the experiments when $M=2$, we limit our focus to MO-CBS-b.

\subsection{MO-CBS High-Level Search}\label{mocbs:sec:result:high-level}

\subsubsection{Success Rates}
We then investigate different high-level search strategies of MO-CBS.
We fix $M=2$ and compare MO-CBS-b (without the tree-by-tree expansion) and MO-CBS-tb (with the tree-by-tree expansion).
We test both algorithms in four maps of various sizes with varying $N$ ranging from $2$ to $10$ with a step size of $2$.
As shown in Fig.~\ref{mocbs:fig:mocbs_vs_momstar}, MO-CBS-b slightly outperforms MO-CBS-tb in terms of the success rate in all four maps.
An intuitive explanation is that MO-CBS-b generates \emph{all} the root nodes at initialization and inserts them into OPEN for search, which makes the search process more informed.
Different from MO-CBS-b, MO-CBS-tb generates the root nodes in a tree-by-tree manner during the search, and greedily search one tree after another, which makes the search process less informed.
To verify the reason, we conduct the following comparison.

\begin{figure}[tb]
	\centering
	\includegraphics[width=\linewidth]{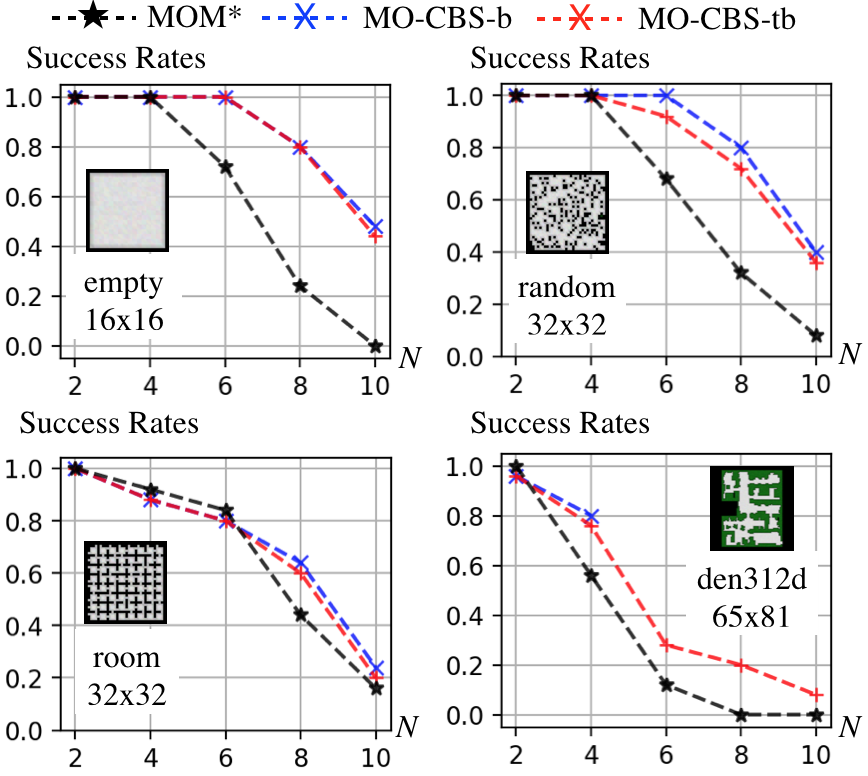}
	\caption{Success rates of MOM* (baseline), MO-CBS-b (this work) and MO-CBS-tb (this work) in four maps of different sizes. MO-CBS based algorithms outperforms the baseline in general. The maximum enhancement of the success rate (around 60\%) can be observed at $N=8$ in the empty map.
	In den312d when $N\geq6$, MO-CBS-b runs out of memory at initialization for some of the instances and is thus omitted.}
	\label{mocbs:fig:mocbs_vs_momstar}
\end{figure}

\subsubsection{Number of Conflicts and Filtered Nodes}
First, we show in Fig.~\ref{mocbs:fig:mocbs_n_conflict} the statistics about the numbers of conflicts (\#Conflict) resolved by both algorithms ($i.e.$ count the times when Alg.~\ref{alg:mocbs} reaches line 11) over all instances.
We can observe that MO-CBS-b in general needs to resolve less conflicts than MO-CBS-tb, which indicates the search process of MO-CBS-b is more efficient than the one of MO-CBS-tb (given that MO-CBS-b has higher success rates).

\begin{figure}[htb]
	\centering
	\includegraphics[width=\linewidth]{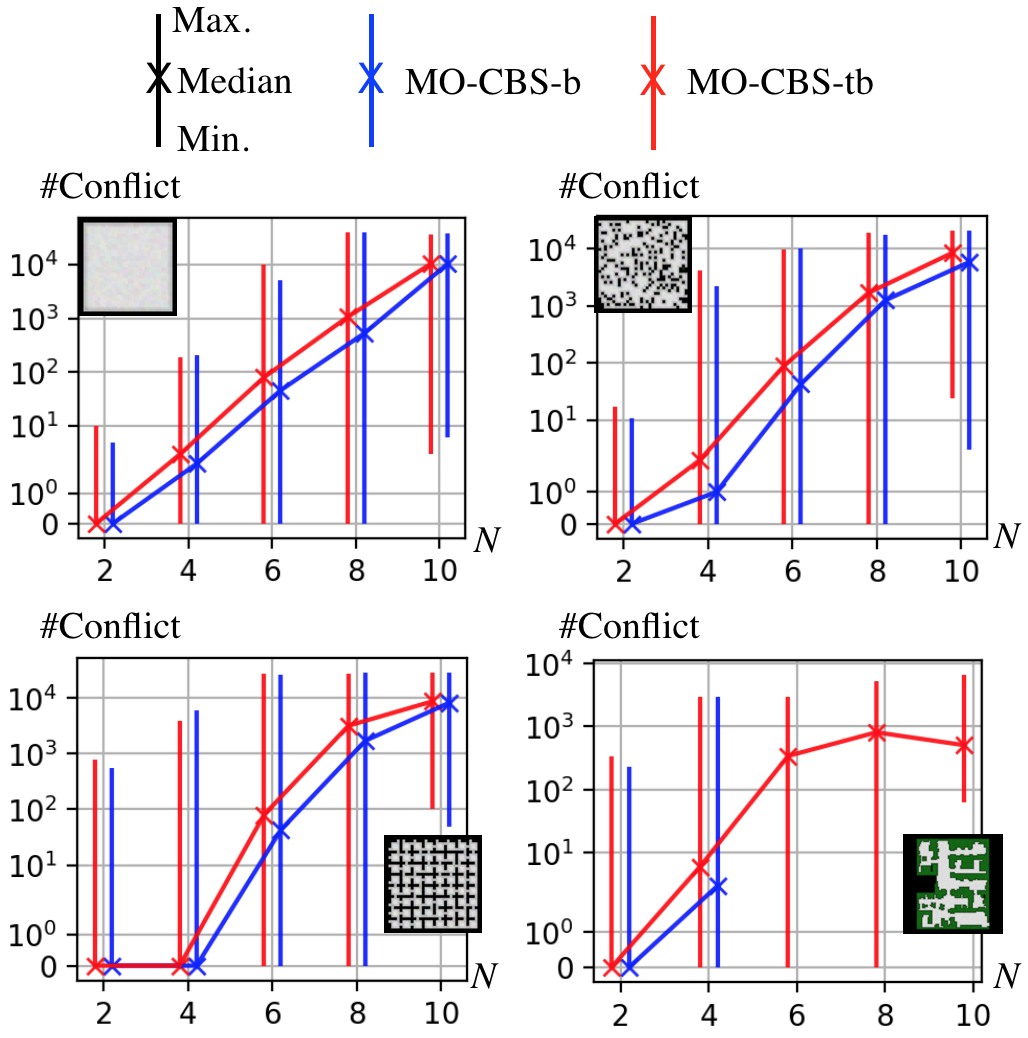}
	\caption{A comparison of the minimum, median and maximum number of conflicts resolved per instance by MO-CBS-b and MO-CBS-tb in various maps.
	In general, MO-CBS-tb needs to resolve more conflicts than MO-CBS-b during the search.
	}
	\label{mocbs:fig:mocbs_n_conflict}
\end{figure}


Second, we look at the statistics about the number of filtered nodes of both algorithms.
The number of filtered nodes (\#Filter) is defines as the times when the \textit{Filter} procedure returns true (line 6 and 20), which means a candidate node is discarded.
In the map random 32x32, as shown in Table.~\ref{mocbs:tab:high_level_solFil}, MO-CBS-b tends to filter fewer nodes than MO-CBS-tb.
Combined with Fig.~\ref{mocbs:fig:mocbs_n_conflict}, we can observe that, with similar success rates, MO-CBS-b resolves less conflicts and filters less nodes than MO-CBS-b does, which indicates that MO-CBS-b can search more efficiently than MO-CBS-tb in general.

\begin{table}[tb]
	\centering
		\tabcolsep=0.1cm
	\renewcommand{\arraystretch}{1.2}
	\begin{tabular}{ | l | l | l |  }
		\hline
		& \multicolumn{2}{c}{min./median/max. \#Filter} \vline
		\\ \hline
		N & MO-CBS-b & MO-CBS-tb
		\\ \hline
		2 & {0 / 2 / 74} & {0 / 6 / 118}
		\\ \hline
		4 & {0 / 51 / 18052} & {1 / 111 / 33051}
		\\ \hline
		6 & {3 / 578 / 74390} & {10 / 1358 / 72175}
		\\ \hline
		8 & {4 / 8881 / 93308} & {13 / 22050 / 116700}
		\\ \hline
		10 & {0 / 23698 / 122478} & {1333 / 41222 / 254374}
		\\ \hline
	\end{tabular}
	\caption{The minimum, median and maximum of the number of filtered nodes per instance (\#Filter) by MO-CBS-b and MO-CBS-tb in the random 32x32 map with varying $N$.
	In general, MO-CBS-b filters less number of nodes than MO-CBS-tb.
	Combined with Fig.~\ref{mocbs:fig:mocbs_n_conflict}, with similar success rates, MO-CBS-b resolves less conflicts and filters less nodes than MO-CBS-b does, which indicates that MO-CBS-b can search more efficiently than MO-CBS-tb in general.}
	\label{mocbs:tab:high_level_solFil}
\end{table}

\subsubsection{Memory Issue}
Although MO-CBS-b can search more efficiently than MO-CBS-tb in general, MO-CBS-b have to generate all the root nodes for initialization, which can consume a lot of memory.
In the den312d map when $N=6$, MO-CBS-b runs out of the 16GB memory and fails to initialize for some of the instances, while MO-CBS-tb bypasses this memory issue due to the tree-by-tree expansion strategy.
As shown in Table~\ref{mocbs:tab:den312d_root}, when $N\geq6$, the number of roots grows up to millions, which makes MO-CBS-b run out of memory to initialize all the root nodes.

\begin{table}[tb]
	\centering
		\tabcolsep=0.1cm
	\renewcommand{\arraystretch}{1.2}
	\begin{tabular}{ | l | l | l | l | l | l |  }
		\hline
		(den312d) & \multicolumn{5}{c}{min./median max. \#Root} \vline 
		\\ \hline
		N & 2 & 4 & 6 & 8 & 10
		\\ \hline
		min. & 7 & 42 & 504 & 2,160 & 11,664
		\\ \hline
		median & 36 & 1,920 & 39,600 & 2,566,080 & 136,080,000
		\\ \hline
		max & 216 & 24,948 & 2,649,536 & 558,379,008 & 17,868,128,256
		\\ \hline
	\end{tabular}
	\caption{The statistics of the number of root nodes (\#Root) of all instances in the den312d 65x81 map with varying $N$.
	As $N$ increases, \#Root grows too large for MO-CBS-b to initialize with the 16GB RAM memory, while MO-CBS-tb bypasses this issue due to the tree-by-tree expansion strategy.}
	\label{mocbs:tab:den312d_root}
\end{table}

\subsubsection{Number of Pareto-optimal Solutions}
This section reports the statistics of the number of Pareto-optimal solutions (\#Sol) and the number of root nodes (\#Root) over all succeeded instances ($i.e.$ all Pareto-optimal solutions are found) in the map random 32x32.
As shown in Table~\ref{mocbs:tab:high_level_n_sol}, both the number of root nodes and the number of Pareto-optimal solutions grow as $N$ increases.\footnote{When $N=8,10$, since the success rates are not $100\%$, there is a bias towards easy instances that have fewer Pareto-optimal solutions. When $N=2,4,6$, since the success rates are $100\%$, there is no such a bias.}
The discrepancy between the \#Root and \#Sol indicates that a large number of root nodes are filtered instead of leading to Pareto-optimal solutions.
It implies a possible future work direction: one can develop new methods for the initialization step of MO-CBS to improve the computational efficiency.

\begin{table}[h]
	\centering
		\tabcolsep=0.1cm
	\renewcommand{\arraystretch}{1.2}
	\begin{tabular}{ | l | l | l |  }
		\hline
		& \multicolumn{2}{c}{min./median/max. \#Sol and \#Root} \vline
		\\ \hline
		N & \#Sol & \#Root
		\\ \hline
		2 & {1 / 5 / 14} & {1 / 6 / 44}
		\\ \hline
		4 & {2 / 9 / 20} & {2 / 30 / 720}
		\\ \hline
		6 & {5 / 13 / 24} & {8 / 249 / 5292}
		\\ \hline
		8 & {6 / 16 / 28} & {10 / 792 / 51840}
		\\ \hline
		10 & {11 / 16 / 29} & {140 / 804 / 97200}
		\\ \hline
	\end{tabular}
	\caption{The minimum, median and maximum of (i) the number of Pareto-optimal solutions (\#Sol) and (ii) the number of root nodes (\#Root) per instance in MO-CBS-b.
	The statistics are computed over all succeeded instances ($i.e.$ all Pareto-optimal solutions are found).}
	\label{mocbs:tab:high_level_n_sol}
\end{table}

\subsection{MO-CBS and MOM*, $M=2,3$}
\subsubsection{Two Objectives}
We compare MO-CBS-b, MO-CBS-tb and MOM* with $M=2$ in four different maps.
When $M=2$, as shown Fig.~\ref{mocbs:fig:mocbs_vs_momstar}, MO-CBS-b and MO-CBS-tb both achieves higher success rates than MOM* in general.
The maximum enhancement of the success rate (around 60\%) can be observed at $N=8$ in the empty map.
In the room map, MOM* has slightly higher success rates than the MO-CBS based algorithms.
In general, it's not obvious under what circumstances MO-CBS is guaranteed to outperform MOM*.
Empirically, there is no leading algorithm that outperforms the other method in all settings.
More discussion can be found in the following paragraphs.

\begin{figure}[tb]
	\centering
	\includegraphics[width=\linewidth]{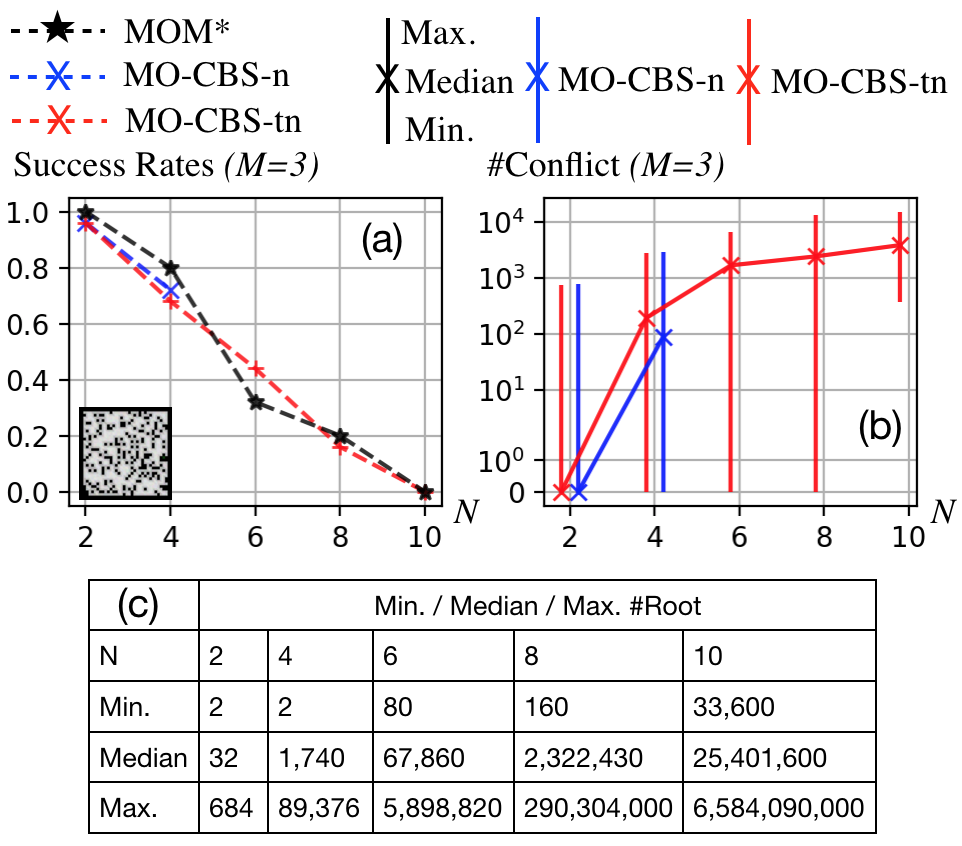}
	\caption{Comparison among MO-CBS-n, MO-CBS-tn, MOM* with $M=3$ (fixed) and varying $N$ in the random 32x32 map.  When $M$ increases from 2 to 3, the problem instances become more challenging and the success rates decrease for all three planners.}
	\label{mocbs:fig:mocbs_vs_momstar_m3}
\end{figure}

\subsubsection{Three Objectives}
We also compare MO-CBS-n, MO-CBS-tn and MOM* with $M=3$ in the random map.
As shown in Fig.~\ref{mocbs:fig:mocbs_vs_momstar_m3}, all three planners achieve similar success rates, which is lower than the corresponding success rates when $M=2$ in Fig.~\ref{mocbs:fig:mocbs_vs_momstar}.
MO-CBS-n fails to initialize due to the large number of root nodes for some instances when $N\geq 6$ and is thus omitted.
In general, when $M$ increases from 2 to 3, the problem becomes more challenging and the success rates decrease for all three planners.

\subsubsection{Discussion}

Since MOM* and MO-CBS are two algorithms that search over different spaces, it's not obvious when one planner is guaranteed to outperform the other.
Intuitively speaking, MOM* searches in the joint graph ($i.e.$ the Cartesian product of individual graphs) with a varying branching factor that is determined by the ``collision set''~\cite{ren2021subdimensional}, the subset of agents that are in conflict.
MO-CBS searches in a different space by detecting and splitting conflicts between agents and the number of conflicts is the decisive factor of the computational efficiency of MO-CBS.

Additionally, in MAPF, all edges are often associated with the \emph{same} unit scalar cost \cite{sharon2015conflict,stern2019multi}, while in MOMAPF, all edges are associated with different cost vectors.
This makes it hard to predict the difficulty of an instance for MOM* or MO-CBS by only looking at the topology of the map without investigating the cost structure.
From our experimental results, one possible indicator about the difficulty of an instance for MO-CBS is the number of root nodes, which reflects the number of individual Pareto-optimal paths of agents, and takes both the topology and the cost structure of the map into consideration.

\subsection{Discussion: MO-CBS and CBS}\label{mocbs:sec:mocbs_vs_cbs}
While (single-objective) CBS can solve up to 21 agents in an empty $8\times8$ four-connected grid and up to hundred of agents in large maps~\cite{sharon2015conflict}, MO-CBS can solve obviously fewer agents for the following reasons.
\textbf{First}, the low-level planner of MO-CBS solves a single-agent multi-objective path planning problem, which is computationally more expensive than the single-objective path planning problem solved by the low-level planner of CBS, especially when there are many individual Pareto-optimal paths to find (as discussed in Sec.~\ref{mocbs:sec:result:low-level}).
\textbf{Second}, in CBS, each agent has only \emph{one} individual optimal path, and CBS terminates when all the conflicts along those individual paths are resolved. In MO-CBS, each agent has \emph{multiple} individual Pareto-optimal paths, and MO-CBS needs to resolve all the conflicts for \emph{any possible combination} of the individual paths, which is computationally much more expensive (Fig.~\ref{mocbs:fig:mocbs_n_conflict}). In other words, the high-level search of MO-CBS needs to search over multiple trees rather than searching a single tree as CBS does (Sec.~\ref{mocbs:sec:result:high-level}).
\textbf{Third}, for conventional CBS, larger maps often lead to less conflicts between agents which allows CBS to handle a large number of agents. For MO-CBS, larger maps can lead to a larger number of individual Pareto-optimal paths (Table.~\ref{mocbs:tab:low_level_maps}), which then slows down the low-level planner and leads to more potential conflicts to be resolved by MO-CBS.

\subsection{Construction Site Path Planning}

This section demonstrates an application example of MO-CBS for practitioners.
We consider multiple agents transporting materials in a construction site \cite{lam2020exact,sartoretti2019distributed,soltani2004fuzzy}.
We focus on planning collision-free paths for a set of agents from their starts to goals while optimizing both the sum of individual arrival times and the sum of individual path risks.
We use a simplified risk model as shown on the left in Fig.~\ref{fig:risk_map}.
We select a random 32x32 map from~\cite{stern2019multi} and compute the corresponding risk map as shown in Fig.~\ref{fig:risk_map} as follows.
The risk score of each cell equals one plus the number of black cells in the 8 neighbors around it, where the black cells represent some semi-constructed architecture.
The risk here is possibly due to the falling items from the architecture or the collision with the architecture.
Similar to the previous tests, each agent can either wait or move to one of the four cardinal adjacent cells.
Each action of the agent incurs a cost vector of length two, where the first component indicates the action time which is always one, and the second component is the risk cost of the arrival cell as aforementioned.
If an agent waits in a cell, the risk cost incurred is the risk score of that cell.

\begin{figure}[tb]
	\centering
	\includegraphics[width=0.95\linewidth]{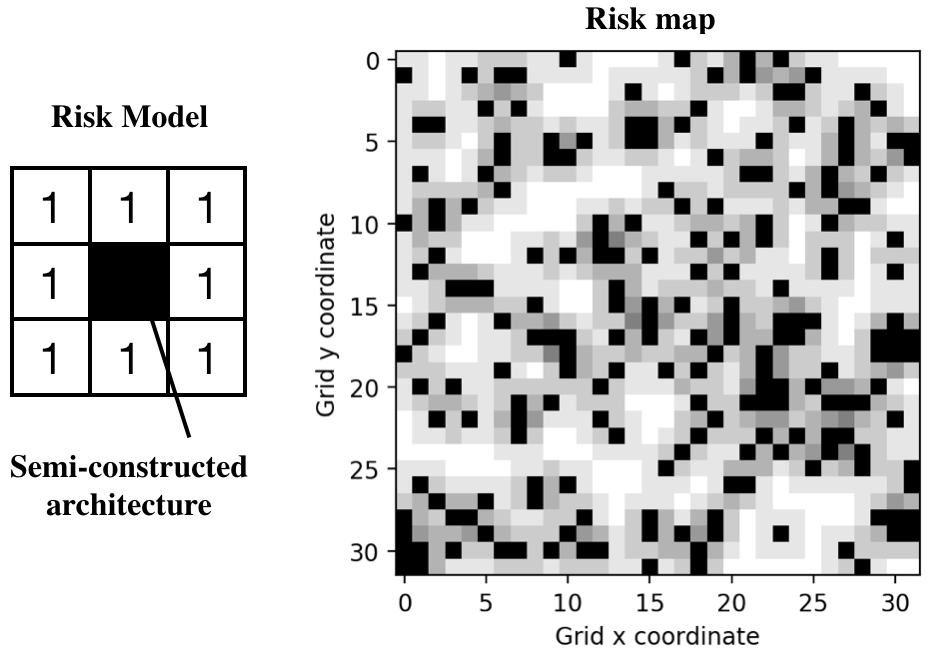}
	\caption{(Left) Risk model. (Right) A risk map where black cells represent semi-constructed architecture and the darkness of a grey cell indicates the risk score of that cell. More details can be found in the text.}
	\label{fig:risk_map}
\end{figure}

\begin{figure*}[tb]
	\centering
	\includegraphics[width=\linewidth]{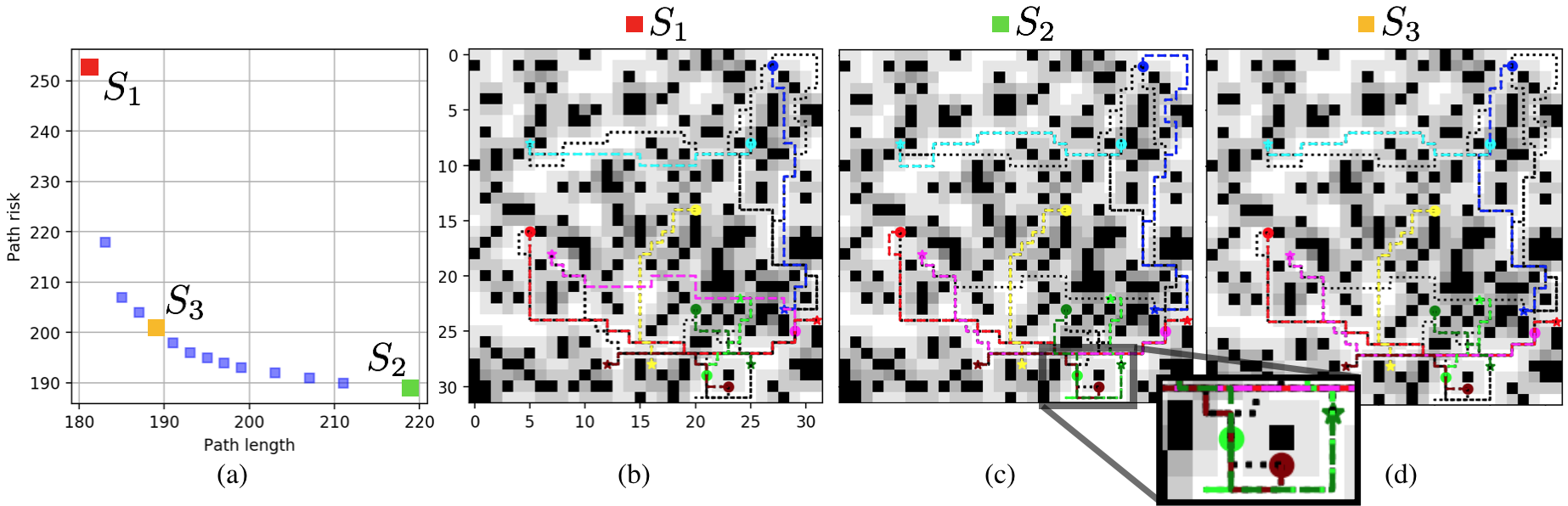}
	\caption{(a) shows the Pareto-optimal front of the construction site example.
	(b), (c) and (d) show three Pareto-optimal solution joint paths corresponding to the red, green and orange solution in (a) respectively.
	In (b), (c) and (d), the colored dotted paths show the individual paths that constitute the corresponding joint path, while the black dotted paths show the individual paths in other Pareto-optimal solutions.
	For solution $S_1$, all agents take shortcut and go through risky zones while for solution $S_2$, all agents are being conservative and go through safe zones. The solution $S_3$ balances the two objectives.
	Finding and visualizing a Pareto-optimal set of solutions can potentially help the human decision maker to understand the underlying trade-off between conflicting objectives and thus make more informed decisions.}
	\label{fig:solution_construction}
\end{figure*}

As shown in Fig.~\ref{fig:solution_construction} (a), the set of Pareto-optimal solutions trades off between the arrival time and path risk.
In solution (joint path) $S_1$ (Fig.~\ref{fig:solution_construction} (b)), all agents take shortcuts regardless of the risks.
For example, the blue agent in $S_1$ passes through many risky cells by following a shortest path.
In solution $S_2$ (Fig.~\ref{fig:solution_construction} (c)), all agents follow the safest paths. For example, in the lower right corner of $S_2$, the light green agent takes a detour to avoid the brown agent to make both of them safe along their respective paths.
The solution $S_3$ (Fig.~\ref{fig:solution_construction} (d)) visualizes a Pareto-optimal solution in the ``middle'', where arrival time and path risk are balanced in some way. 


\section{Conclusion and Future Work}\label{sec:conclude}

In this article, we develop a new algorithm called Multi-Objective Conflict-Based Search (MO-CBS) to solve Multi-Objective Multi-Agent Path Finding (MOMAPF) problems with optimality guarantees.
We also develop several variants of MO-CBS by using various low-level planners and different high-level expansion strategies.
We analyze the properties of MO-CBS and show that the method is able to find all cost-unique Pareto-optimal solutions.
Numerical results show that MO-CBS outperforms the baseline in terms of success rates under a runtime limit.
We also show an application example for practitioners.

There are several directions for future work.
One can consider improving MO-CBS by expediting its initialization, conflict splitting and dominance checks.
Additionally, instead of finding an exact set of Pareto-optimal solutions, one can focus on approximating the Pareto-optimal solutions with faster computational speed and better scalability, in terms of both the number of agents and the number of objectives.
For example, one can consider leveraging evolutionary algorithms~\cite{emmerich2018tutorial} or approximated single-agent multi-objective planners~\cite{goldin2021approximate} to expedite the computation.
One can also develop other types of MOMAPF algorithms by leveraging other (single-objective) MAPF methods to handle agents that move with different speeds~\cite{andreychuk2022multi,ren2021loosely,weise2021scalable} or targets that need allocation~\cite{honig2018conflict,ren22cbss}.


\bibliographystyle{plain}
\bibliography{references}

\newpage

\begin{IEEEbiography}[{\includegraphics[width=1in,height=1.25in,clip,keepaspectratio]{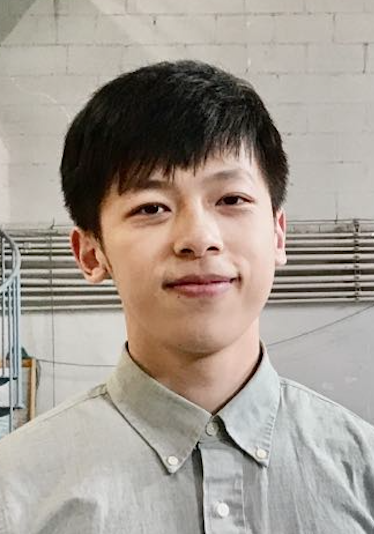}}]{Zhongqiang (Richard) Ren}
(Student Member, IEEE)
received the M.S. degree from Carnegie Mellon University, Pittsburgh, PA, USA, and the dual B.E. degree from Tongji University, Shanghai, China, and FH Aachen University of Applied Sciences, Aachen, Germany. He is currently a Ph.D. candidate at Carnegie Mellon University.
\end{IEEEbiography}

\begin{IEEEbiography}[{\includegraphics[width=1in,height=1.25in,clip,keepaspectratio]{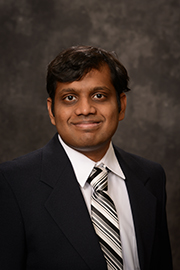}}]{Sivakumar Rathinam}
(Senior Member, IEEE)
received the Ph.D. degree from the University of
California at Berkeley in 2007. He is currently
a Professor with the Mechanical Engineering Department, Texas A\&M University. His
research interests include motion planning and control of autonomous vehicles, collaborative decision
making, combinatorial optimization, vision-based
control, and air traffic control.
\end{IEEEbiography}

\begin{IEEEbiography}[{\includegraphics[width=1in,height=1.25in,clip,keepaspectratio]{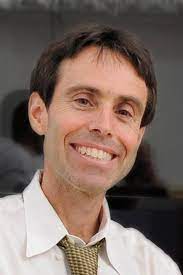}}]{Howie Choset}
(Fellow, IEEE)
received the undergraduate degrees in
computer science and business from the University of
Pennsylvania, Philadelphia, PA, USA, and the M.S.
and Ph.D. degrees in mechanical engineering from
Caltech, Pasadena, CA, USA.\\
He is a Professor in the Robotics Institute, Carnegie
Mellon, Pittsburgh, PA, USA.
\end{IEEEbiography}

\end{document}